\documentclass[letterpaper, 10 pt, conference]{ieeeconf}

\IEEEoverridecommandlockouts
\usepackage{graphicx}
\usepackage{graphicx,charter, latexsym}

\newtheorem{lemma}{Lemma}
\usepackage{amssymb, amsmath,graphicx,charter, latexsym}
\usepackage{gensymb}
\usepackage{tabularx}

\newtheorem{theorem}{Theorem}

\usepackage{multirow,url,float,epstopdf,cite}
\usepackage{verbatim}
\newtheorem{assumption}{Assumption}
\usepackage{graphicx}
\usepackage{wrapfig} 
\usepackage{caption}
\usepackage{subcaption}
\usepackage{algorithmic}
\usepackage{algorithm}
\usepackage{tikz}
\usepackage{pifont}
\usepackage{xcolor}
\usepackage{caption}
\usepackage{cite,url}
\usepackage{bbm}
\usetikzlibrary{patterns,calc}
\usetikzlibrary{shapes.geometric, arrows}
\usepackage[overlay,absolute]{textpos}

\tikzstyle{block}=[draw opacity=0.7,line width=1.4cm]
\tikzstyle{ag} = [circle, radius =3cm, text centered, draw=black]
\tikzstyle{startstop} = [rectangle, rounded corners, minimum width=3cm, minimum height=1cm,text centered, draw=black, fill=red!30]
\tikzstyle{io} = [trapezium, trapezium left angle=70, trapezium right angle=110, minimum width=1cm, minimum height=1cm, text centered, draw=black, fill=blue!30]
\tikzstyle{process} = [rectangle, draw,fill=orange!30, text width = 20em, text centered, rounded corners, minimum height=4em, minimum width=1cm]
\tikzstyle{upd} = [rectangle, draw,fill=orange!30, text width = 5em, text centered, rounded corners, minimum height=4em, minimum width=1cm]
\tikzstyle{decision} = [diamond, minimum width=3cm, minimum height=1cm, text centered, draw=black, fill=green!30]
\tikzstyle{arrow} = [thick,->,>=stealth]
\tikzset{
  solid node/.style={circle,draw,inner sep=1.2,fill=black},
  hollow node/.style={circle,draw,inner sep=1.2},
}

\newcommand{\bZ}{\mathbb{Z}}
\newcommand{\cX}{\mathcal{X}}
\newcommand{\bN}{\mathbb{N}}
\newcommand{\bR}{\mathbb{R}}
\newcommand{\bP}{\mathbb{P}}
\newcommand{\bE}{\mathbb{E}}
\newcommand{\cC}{\mathcal{C}}
\newcommand{\cT}{\mathcal{T}}
\newcommand{\cF}{\mathcal{F}}
\newcommand{\cH}{\mathcal{H}}
\newcommand{\cU}{\mathcal{U}}

\newcommand{\cE}{\mathcal{E}}
\newcommand{\cG}{\mathcal{G}}
\newcommand{\cJ}{\mathcal{J}}

\newcommand{\ust}{_{(\star)}}
\newcommand{\tht}{_{\theta(t)}}

\DeclareTextFontCommand{\emph}{\em}

\IEEEoverridecommandlockouts

\title{Learning in Networked Control Systems}

\author{Rahul Singh and P. R. Kumar
\thanks{Rahul Singh is with the Department of Electrical and Computer Engineering,
        The Ohio State University, Columbus, OH, 43210, USA
        {\tt\small rahulsiitk@gmail.com}}%
\thanks{P. R. Kumar is with the Department of Electrical \& Computer Engineering, Texas A\&M University, College Station, TX 77843, USA
        {\tt\small prk.tamu@gmail.com}}%
}

\begin{document}
\maketitle
\begin{abstract}
We design adaptive controller (learning rule) for a networked control system (NCS) in which data packets containing control information are transmitted across a lossy wireless channel.
We propose Upper Confidence Bounds for Networked Control Systems (UCB-NCS), a learning rule that maintains confidence intervals for the estimates of plant parameters $(A\ust,B\ust)$, and channel reliability $p\ust$, and utilizes the principle of optimism in the face of uncertainty while making control decisions.\par
We provide non-asymptotic performance guarantees for UCB-NCS by analyzing its ``regret", i.e., performance gap from the scenario when $(A\ust,B\ust,p\ust)$ are known to the controller. We show that with a high probability the regret can be upper-bounded as $\tilde{O}\left(C\sqrt{T}\right)$\footnote{Here $\tilde{O}$ hides logarithmic factors.}, where $T$ is the operating time horizon of the system, and $C$ is a problem dependent constant. 
\end{abstract}
\section{Introduction}
Though adaptive control~\cite{bellman2015adaptive} of unknown Linear Quadratic Gaussian (LQG) systems~\cite{kumar} is a well-studied topic by now~\cite{becker1985adaptive,chen1987optimal,bittanti2006adaptive,abbasi2011regret}, existing algorithms cannot be utilized for controlling an unknown NCS in which plant and network parameters are unknown. In departure from the traditional adaptive controllers for LQG systems, an algorithm now also needs to continually estimate the unknown network behaviour besides simultaneously learning and controlling the plant in an \emph{online} manner. An \emph{important} concern is that in general it is not optimal to design and operate network estimator \emph{independently} of the process controller. Thus, the optimal controls $u(t)$ should utilize the information gained about network quality in addition to using the information gained about plant parameters. Similarly, decisions made by the network scheduler should also ``aid" the controller in ``learning" the unknown plant parameters.\par 
This work addresses the problem of adaptive control of a simple NCS in which data packets from the controller to the plant, are communicated over an unreliable channel. We model the plant as a LQG system. We propose a learning rule that maintains estimates and confidence sets for both a) (unknown) plant parameters $(A\ust,B\ust)$, and also b) (unknown) channel reliability $p\ust$. Controls are then generated using the principle of optimism in face of uncertainty~\cite{lai1985asymptotically}, and depend upon both a) and b). We denote our algorithm as Upper Confidence Bounds for Networked Control Systems (UCB-NCS).\par
We show that UCB-NCS yields the same asymptotic performance as the optimal controller that has knowledge of the system and network parameters. We also quantify its finite-time performance by providing upper-bounds on its ``regret"~\cite{auer2002finite}. Regret scales as $\tilde{O}\left(C\sqrt{T}\right)$, where $T$ is the operating time horizon and $C$ is a problem dependent constant. It also depends on the channel reliability through a certain quantity which we call the ``margin of stability" $\eta$~\eqref{def:eta}. A larger value of $\eta$ means that the learning algorithm has a lower regret.\par
UCB-NCS has many appealing properties. For instance, network estimator needs to communicate only occasionally the value of its optimistic estimate of network reliability to the controller which then uses it to generate controls. 
 \section{System Model}
We assume that the system of interest is linear, and evolves as follows
\begin{align}\label{eq:system_dynamics}
x(t+1) = 
\begin{cases}
A\ust x(t) + B\ust u(t) + w(t) \mbox{ if } \ell(t) = 1 \\
A\ust x(t) + w(t) \mbox{ if } \ell(t) = 0,
\end{cases}
\end{align}  
where $A\ust\in\bR^{n\times n},B\ust\in \bR^{n\times m}$ are the system matrices, $\ell(t)\in \left\{0,1\right\}$ is the instantaneous state of the wireless channel, and $x(t)\in\bR^{n}, u(t)\in \bR^{m}$ are the system state and control input at time $t$ respectively. $\{\ell(t)\}_{t=1}^{T}$ are Bernoulli i.i.d. with mean value $p\ust$. $\{w(t)\}_{t=1}^{T}$ is the process noise, and is assumed to be i.i.d. with 
\begin{align*}
\bE \left( w(t)w^{T}(t) \right)= \sigma^{2}_{w},~\forall t\in [1,T].
\end{align*} 
The objective is to minimize the operating cost
\begin{align}\label{def:cost}
\bE  \sum_{t=1}^{T-1} x^{T}(t)Q x(t) + u^{T}(t)R u(t) + x^{T}(T)Q x(T).
\end{align}
We let $\theta\ust := \left(A\ust,B\ust,p\ust\right)$ denote the system parameters. $\theta\ust$ is not known to controller. We assume that the system is scalar, i.e., $m=n=1$.
\section{Preliminaries on Jump Markov Linear Systems}\label{sec:jmls}
Note that~\eqref{eq:system_dynamics} is a Jump Markov Linear System (JMLS), and if the system parameter $\theta\ust$ is known, the optimal controls can be obtained by using Dynamic Programming~\cite{costa2006discrete}.

There are matrices $\left\{K_{\theta\ust}(\ell)\right\}_{\ell\in \{0,1\} }$ such that the optimal control at $t$ is given by $K_{\theta\ust}(\ell(t)) x(t)$. We let $\left\{K_{\theta}(\ell)\right\}_{\ell\in \{0,1\} }$ denote the optimal matrices when system parameter is equal to $\theta$. 

We let $V_{\theta}(x,\ell)$ denote the ``cost-to-go" when system state is equal to $x$, channel state is $\ell$ and system dynamics are described by $\theta$. In fact value function is piecewise linear, and we let $\{P_{\theta}(\ell)\}_{\ell\in\{0,1\}}$ denote the corresponding matrices. We also let $J_{\theta}$ be the optimal operating cost.

\emph{Notation}: For a random variable (r.v.) $X$, let $X_{\cF}$ denote its projection onto the space of $\cF$ measurable funcions, i.e., its conditional expectation w.r.t. sigma-algebra $\cF$. For $x,y\in \bZ$\footnote{$\bZ$ denotes the set of integers.}, we let $[x,y]:= \left\{x,x+1,\ldots,y\right\}$. For a set of r.v. s $\cX$, we let $\sigma(\cX)$ denote the smallest sigma-algebra with respect to which each r.v. in $\cX$ is measurable. For functions $f(x),g(x)$, we say $f(x)=O(g(x))$ if $\lim_{x\to\infty} f(x)\slash g(x)=1$. For a set $\cX$, we let $\cX^{c}$ denote its complement.
\section{Upper Confidence Bounds for NCS (UCB-NCS)}
Let $\cF_t:= \sigma\left(\left\{ (x(s),u(s))\right\}_{s=1}^{t-1}\cup \{x(t)\} \right)$. A learning policy, or an adaptive controller is a collection of maps $\left\{\cF_t\mapsto u(t)\right\}_{t=1}^{T}$. Let $\hat{\theta}(t):= \left(\hat{A}(t),\hat{B}(t),\hat{p}(t),\right)$ denote the estimates of $\theta\ust = (A\ust,B\ust,p\ust)$ at time $t$ defined as follows. Let $z(s) := x(s+1)$, and $\lambda >0$.
\begin{align}\label{def:point_est}
&\hat{p}(t) =\sum_{s=1}^{t} \ell(s) \slash t,\notag\\
&\hat{A}(t) \in \arg\min 1\slash 2\left[ \lambda A^2 + \sum_{s=1}^{t-1} \left( z(s) -A x(s) \right)^2 (1-\ell(s))\right],\notag\\
&\hat{B}(t) \in \notag\\
&\arg\min \left[\frac{\lambda B^2}{2} + \frac{\sum_{s=1}^{t-1} \left( z(s) -\hat{A}(t) x(s) - B u(s) \right)^2 \ell(s)}{2}\right],
\end{align}
Define
\begin{align}\label{def:gamma}
 V_{1}(t):&= \lambda + \sum_{s=1}^{t-1} x^2(s)(1- \ell(s) ),V_2(t):= \lambda + \sum\limits_{s=1}^{t-1} u^2(s)\ell(s),\notag\\
\gamma_i(\delta,t)&:= \sqrt{\log\left(\lambda V_i(t)\slash \delta \right)},~i=1,2.
\end{align}
Let $\cC(t)=\left( \cC_{1}(t),\cC_{2}(t),\cC_{3}(t)\right)$ be the confidence intervals associated with the estimates $\left( \hat{A}(t),\hat{B}(t),\hat{p}(t)\right)$ at time $t$ defined as follows,  
\begin{align}\label{def:ucb_ci} 
\cC_{1}(t) :&= \left\{ A :  |A - \hat{A}(t)| \le \beta_1(t)  \right\}, \\
\cC_{2}(t) :&= \left\{ B :  |B - \hat{B}(t)| \le \beta_2(t)  \right\}, 
\notag\\
\cC_{3}(t) :&= \left\{ p:  |p -\hat{p}(t) | \le  \beta_3(t) \right\},
\end{align} 
where
\begin{align*}
\beta_1(t) : &=(\gamma_1(\delta,t)+\lambda^{1\slash 2})\slash \sqrt{V_{1}(\delta,t)}, ~\beta_3(t) : = \sqrt{\log\left(1\slash \delta\right) \slash t}\\
\beta_2(t) : &=  \frac{(\gamma_2(\delta,t)+\lambda^{1\slash 2})}{ \sqrt{V_{2}(t)}} + K_{\max}\frac{(\gamma_1(\delta,t)+\lambda^{1\slash 2})}{\sqrt{V_{1}(\delta,t)}}.
\end{align*}

The learning rule decomposes the cumulative time into episodes, and implements a single stationary controller within each single episode that chooses $u(t)$ as a function of $x(t)$. Let $\tau_k$ denote the starting time of $k$-th episode. The controller implemented within episode $k$ is obtained at time $\tau_k$ by solving the following optimization problem.
\begin{align}\label{def:ucb_problem}
\min_{\theta\in \cC(\tau_k)\cap \Theta } J_{\theta},
\end{align}
where $\Theta$ is the set of ``allowable" parameters. 
Let $\theta(\tau_k)$ denote a solution to above problem. It implements the optimal controller corresponding to the case when true system parameters are equal to $\theta(\tau_k)$. $u(t)= K_{\theta(\tau_k)}(\ell(t))x(t)$. Thus, $u(t)=K_{\theta(\tau_k)}(\ell(t))x(t)$ for $t\in\left[\tau_k,\tau_{k+1}-1\right]$.\par 
A new episode begins when either $V_1(t)$ or $V_2(t)$ doubles or the operating time spent in current episode becomes equal to length of previous episode. The learning rule also ensures that the durations of episodes are at least $L$ time-slots, i.e., $\tau_{k+1}-\tau_k \ge L$.
We set 
\begin{align*}
\theta(t):= \theta(\tau_k), \forall t\in \left[\tau_k,\tau_{k+1}-1\right],
\end{align*}
i.e., it is the current value of the UCB estimate of $\theta\ust$. UCB-NCS is summarized in Algorithm~\ref{algo:ucb}.
\begin{algorithm}
 \caption{UCB-NCS}
 \begin{algorithmic}[1]\label{algo:ucb}
 \renewcommand{\algorithmicrequire}{\textbf{Input:}}
 \REQUIRE $T,\lambda>0,\delta>0,L\in \bN,\alpha>2$\\
 Set $V^{1,\star},V^{2,\star}=\lambda,\hat{A}(1)=.5,\hat{B}(1)=.5,\hat{p}(1)=.5,\tau=1,V_1(1)=\lambda,V_2(1)=\lambda$.
  \FOR {$t=1,2,\ldots$} 
  \IF {($V_1(t)\ge 2V^{1,\star}$ or $V_2(t)\ge 2V^{2,\star}$ or $t\ge 2\tau$)
  and $t-\tau\ge L$}
  \STATE Calculate $\hat{\theta}(t)$ as in~\eqref{def:point_est} and $\theta(t)$ by solving~\eqref{def:ucb_problem}. \\
  Update $V^{1,\star}=V_1(t),V^{2,\star}=V_2(t),\tau=t$
  \ELSE 
  \STATE $\hat{\theta}(t)=\hat{\theta}(t-1)$
  \ENDIF\\
  Calculate $u(t)$ based on current UCB estimate $\theta(t)$, system state $x(t)$, and channel state $\ell(t)$. Use control $u(t)=K_{\theta(t)}(\ell(t))x(t)$.\\
  Update $V_1(t+1)=V_1(t)+x^2(t)(1-\ell(t)),V_2(t+1)=V_2(t)+u^{2}(t)\ell(t)$
  \ENDFOR
 \end{algorithmic} 
 \end{algorithm}
\section{Large Deviation Bounds on Estimation Errors}
We now analyze the estimation errors $e_1(t):=\hat{A}(t) - A, e_2(t):=\hat{B}(t) - B $. 
\begin{lemma}
Define 
\begin{align*}
\cE := \left\{ \omega: \theta\ust=\left(A\ust,B\ust,p\ust\right) \in \cC(t), ~\forall t\in [1,T]  \right\}.
\end{align*}
We then have that
\begin{align*}
\bP\left( \cE^{c}  \right) \le 3\delta.
\end{align*}
\end{lemma}
\begin{proof}
It can be shown that
\begin{align}\label{def:hat_a}
e_1(t)= -\lambda A\slash V_1(t) + \sum\limits_{s=1}^{t-1} w(s)x(s)(1- \ell(s) )\slash V_1(t).
\end{align}
Note that $\{w(s)\}_{s=1}^{T-1}$ is a martingale difference sequence w.r.t. $\cF_t$, while $x(t)$ is adapted to $\cF_t$. Thus, bound on $e_1(t)$ follows by using self-normalized bounds on martingales from Corollary 1 of~\cite{abbasi2011online}.    

To analyze $e_2(t)$, we observe,
\begin{align}\label{def:hat_b}
&e_2(t) = \left(\sum\limits_{s=1}^{t-1} w(s)u(s)\ell(s)\slash V_2(t) - \lambda B\slash V_2(t) \right)\notag\\
&+[A-\hat{A}(t)] \sum\limits_{s=1}^{t-1}x(s)u(s) \ell(s)\slash V_2(t).
\end{align}
The first term within braces is bounded using Corollary 2 of~\cite{abbasi2011online}. To bound the second term, we observe that it is upper-bounded by $K_{\max} |e_1(t)|$. We then use bounds on $e_1(t)$ to bound it. Bound on estimation error of $p\ust$ is obtained using Azuma-Hoeffding inequality.
\end{proof}

\section{Large Deviation Bounds on the System State $|x(t)|$}
We now bound $|x(t)|$ under UCB-NCS. System evolution under UCB-NCS is given by
\begin{align*}
x(t+1) = A_{sw}(t) x(t)+ w(t), t\in [1,T-1],
\end{align*}
where 
\begin{align*}
A_{sw}(t):= \left[ \left( A\ust + B\ust K_{\tht}(\ell(t)) \right) \ell(t)  + A\ust (1-\ell(t)) \right].
\end{align*}
Thus,
\begin{align}\label{eq:sw_sys_dyn}
x(t) = x(0)G(0,t) + \sum_{s=1}^{t-1} w(s)G(s,t-1),
\end{align}
where
\begin{align*}
G(s_1,s_2):= 
\begin{cases}
\prod\limits_{\ell=s_1}^{s_2} A_{sw}(\ell) \mbox{ if } s_2 > s_1,\\
1 \mbox{ if } s_1 = s_2.
\end{cases} 
 \end{align*}
Consider the deviations 
\begin{align*}
\Delta(t_1,t_2):= \sum\limits_{s=t_1}^{t_2}\ell(s) - p\ust(t_2-t_1),
\end{align*}
and the events,
\begin{align}\label{ineq:ah}
\cJ_{t_1,t_2} :=  \left\{ \omega: |\Delta(t_1,t_2)| \le \sqrt{2\alpha\sigma^2_{p\ust} (t_2-t_1)\log(t_2-t_1)}  \right\},
\end{align}
where $\sigma^2_{p\ust}:= p\ust(1-p\ust)$, and $\alpha>2$. It follows from Azuma-Hoeffding inequality that
\begin{align}
\bP\left( \cJ^{c}_{t_1,t_2}\right) \le \frac{1}{(t_2-t_1)^{\alpha}},~ \forall t_1,t_2 \in [1,T].\label{ineq:6}
\end{align}
Fix a sufficiently large $L>0$\footnote{It suffices to let $L>\left(2\alpha \sigma^2_{p\ust}\slash \epsilon^{2}\right)^{2}$}, and define
\begin{align}\label{def:J}
\cJ := \cap_{t_1,t_2: t_2 \ge t_1 + L}~\cJ_{t_1,t_2}. 
\end{align}
The following result by combining union bound with the bound~\eqref{ineq:6}.
\begin{lemma}\label{lemma:p_j}
\begin{align*}
\bP\left(\cJ^{c} \right) \le T^{2}\slash L^{\alpha}.
\end{align*} 
\end{lemma}
We now focus on upper-bounding $|G(s,t)|$ on $\cJ$.\par
Throughout, we assume that the true system parameter $\theta\ust$, and the set $\Theta$ used by UCB-NCS, satisfy the following.
\begin{assumption}\label{assum:clg}
Define
\begin{align*}
\Lambda(\theta):= \bE\left( \log A_{sw}(t) | \theta(t) = \theta\right).
\end{align*}
Let $\epsilon>0,\eta>0$. Then,
\begin{align}\label{def:eta}
 \Lambda(\theta) < -\eta-\epsilon < 0,~\forall \theta \in \Theta.
\end{align}
We call $\eta$ as the ``margin of stability" of the NCS. Note that $\eta$ depends upon a) $\Theta$, b) $(A\ust,B\ust,p\ust)$.
\end{assumption}
Consider an element of $\cJ$, and assume there are $k$ episodes during the time period $[s,t]$. Let $N_{i,k}, i=0,1$ denote the number of times channel state assumes value $i$, and let $\theta_k$ denote the UCB estimate of $\theta\ust$  during the $k$-th episode. Let $D_k$ denote the duration of $k$-th episode. We have the following,
\begin{align}\label{bound:J}
&|G(s,t)| = \prod_{m=s}^{t} A_{sw}(\ell)\notag\\
&\le \prod_{k=1}^{K} \exp\left( D_k\Lambda(\theta_k) \right) \exp\left( \sqrt{2\alpha\sigma^2_{p\ust}D_k \log D_k}  \right) \notag \\
&\le \exp\left(-\eta (t-s)  \right),
\end{align}
where the first inequality follows from definition of $\cJ$~\eqref{def:J}, while the second follows from Assumption~\ref{assum:clg}.

Let
\begin{align*}
\cH := \left\{\omega: \max_{t\in [1,T] }|w(t)| \le \log^{1\slash 2}\left(T\slash \delta \right)   \right\}.
\end{align*}
Following is easily proved.
\begin{lemma}\label{lemma:bound_noise}
We have
\begin{align*}
\bP\left( \cH^{c} \right) \le \delta.
\end{align*}
\end{lemma}

\begin{lemma}\label{lemma:bounded_xt}
Define 
\begin{align}\label{def:upper_bound_expr}
g(\delta,T) :=|x(0)| + \log^{1\slash 2}\left(T\slash\delta
\right)\slash (1-\exp(-\eta)).
\end{align}
Under Assumption~\ref{assum:clg}, we have the following on $\cH \cap \cJ $
\begin{align*}
|x(t)| < g(\delta,T),~\forall t\in [1,T].
\end{align*}
Note that we have suppressed dependence of function $g$ upon $\eta,x(0)$.
\end{lemma}

\begin{proof}
The proof follows by substituting in~\eqref{eq:sw_sys_dyn} the bound~\eqref{bound:J} on $|G(s,t)|$ and the bound $\log^{1\slash 2}\left(\frac{T}{\delta}\right)$ on $|w(s)|$ on the set $\cH$.
\end{proof}
\section{Regret Analysis of UCB-NCS}
Define $R(T)$, the regret incurred by UCB-NCS until time $T$ as follows
\begin{align}\label{def:regret}
R(T) :&= \sum_{t=1}^{T}  c(t) - T J_{\theta\ust},\notag\\
\mbox{ where } c(t) :&= Qx^2(t) + Ru^2(t).
\end{align}
For $\theta = (A,B,p)$, define
\begin{align*}
x_{\theta}(t+1;u ) = A x(t) + B u + w(t).
\end{align*}
Similarly, let $\{\ell_{\theta}(t)\}_{t=1}^{T}$ be drawn i.i.d. according to $\theta$. 
\begin{lemma}\label{lemma:bellman_regret}
On the set $\cE$, $R(T)$ can be upper-bounded as follows,
\begin{align*}
R(T) \le R_1 + R_2,
\end{align*}
where,
\begin{align*}
R_1 :&= \sum_{t=1}^{T-1} V\tht(x_{\theta\ust}(t+1;u(t)),\ell\ust(t+1))_{\cF_t} \\
&-  V\tht(x(t),\ell(t)) \\
R_2 :&= \sum_{t=1}^{T-1} V\tht(x\tht(t+1;u(t)),\ell\tht(t+1))_{\cF_t}\\
& -
V\tht(x_{\theta\ust}(t+1;u(t)),\ell\ust(t+1))_{\cF_t}.
\end{align*}
\end{lemma}
\begin{proof}
Consider the Bellman optimality equation at time $t$ when the true system parameter is assumed equal to $\theta(t)$,
\begin{align}\label{eq:1}
& J_{\theta(t)} + V\tht(x(t),\ell(t))= Qx^2(t)\notag\\
&+ \min_{u\in \bR} \left[Ru^2+ V\tht(x\tht(t+1;u),\ell\tht(t+1))_{\cF_t}  \right]\notag \\
&= Qx^2(t) + Ru^2(t) + V\tht(x_{\theta\ust}(t+1;u(t)),\ell\ust(t+1))_{\cF_t} \notag \\ 
&+ V\tht(x\tht(t+1;u(t)),\ell\tht(t+1))_{\cF_t}\notag\\
& -
V\tht(x_{\theta\ust}(t+1;u(t)),\ell\ust(t+1))_{\cF_t}
\end{align}
where the second equality follows since the learning rule applies controls by assuming that $\theta(t)$ is the true system parameter. Note that on $\cE$, $J_{\theta(t)}$ serves as a lower bound on the optimal cost $J_{\theta\ust}$ , so that $\sum\limits_{t=0}^{T} \left(Qx^2(t) + Ru^2(t) \right)- \sum\limits_{t=0}^{T} J_{\theta(t)}$ serves as an upper-bound on $R(T)$. Proof is completed by re-arranging the terms in~\eqref{eq:1}, and summing them from $t=1$ to $t=T-1$.
\end{proof}
We now bound the terms $R_1,R_2$ on $\cE$.
\subsection{Bounding $R_1$}
We decompose $R_1$ as follows, $R_1 = \cT_1 + \cT_2$, where,
\begin{align*}
\cT_1 :&= \sum_{t=1}^{T-1} V_{\theta(t-1)}(x_{\theta\ust}(t;u(t-1)),\ell\ust(t))_{\cF_{t-1}} \\
&\qquad\qquad -  V_{\theta(t)}(x(t),\ell(t)), \\
\cT_2 :&= V_{\theta(T-1)}(x_{\theta\ust}(T;u(T-1)),\ell\ust(T))_{\cF_{T-1}} \\
&\qquad\qquad -  V_{\theta(1)}(x(1),\ell(1)).
\end{align*}
We further decompose $\cT_1$ as follows, 
\begin{align*}
\cT_1 = \cT_3 + \cT_4,
\end{align*}
where,
\begin{align*}
\cT_3 :&= \sum_{t=1}^{T-1} V_{\theta(t-1)}(x_{\theta\ust}(t;u(t-1)),\ell\ust(t))_{\cF_{t-1}} \\
&\qquad \qquad \qquad \qquad \qquad-  V_{\theta(t-1)}(x(t),\ell(t)) \\
\cT_4 :&= \sum_{t=1}^{T-1} V_{\theta(t)}(x(t),\ell(t)) - V_{\theta(t-1)}(x(t),\ell(t)).
\end{align*}
\begin{lemma}\label{lemma:stopped_mtgle}
\begin{align*}
\bP\left( \cT_3 > \sqrt{T g(\delta,T)\log\left(T\slash\delta\right)}   \right) \le \delta + \bP\left(  \left[\cH \cap \cJ\right]^{c} \right),
\end{align*}
where $g(\delta,T)$ is as in~\eqref{def:upper_bound_expr}.
\end{lemma}
\begin{proof}
$\cT_3$ is a martingale, though its increments are not bounded. However, its increments are upper-bounded as $O\left(|x(t)|\right)$. 
It follows from Lemma~\ref{lemma:bounded_xt} that its increments are upper-bounded as $O\left(g(\delta,T)\right)$ on $\cH \cap \cJ$. The proof then follows from Proposition 34 of~\cite{tao2011random}. 
\end{proof}
Henceforth denote
\begin{align*}
\cG:= \left\{\omega:  \cT_3 < \sqrt{T g(\delta,T)\log\left(T\slash\delta\right)}  \right\}.
\end{align*}
We obtain the following bound on $R_1$ by combining results of Lemma~\ref{lemma:stopped_mtgle} and Lemma~\ref{lemma:t_4}.  
\begin{lemma}[Bounding $R_1$]\label{lemma:r_1_b}
Let
\begin{align}\label{def:U_1}
\cU_1 :&= \sqrt{T g(\delta,T)\log\left(T\slash \delta\right)}\notag\\
& +  2P_{\max} g^{2}(\delta,T) + P_{\max}  f(\delta,T) g(\delta,T),
\end{align}
where $g(\delta,T),f(\delta,T)$ are as in~\eqref{def:upper_bound_expr},~\eqref{def:f}. On $\cG \cap \left( \cH \cap \cJ \right)$ we have $R_1 \le \cU_1 $.
\end{lemma}
  
\subsection{Bounding $R_2$}
We decompose $R_2$ as follows,
\begin{align}
R_2 = \cT_5 + \cT_6.\label{def:R_2}
\end{align}
where 
\begin{align*}
\cT_5  :&= \sum_{t=1}^{T-1} V\tht(x\tht(t+1;u(t)),\ell\tht(t+1))_{\cF_t}\\
&\qquad \qquad  - V\tht(x_{\theta\ust}(t+1;u(t)),\ell\tht(t+1))_{\cF_t}\\
\cT_6 :&= \sum_{t=1}^{T-1} V\tht(x_{\theta\ust}(t+1;u(t)),\ell\tht(t+1))_{\cF_t}\\
&-V\tht(x_{\theta\ust}(t+1;u(t)),\ell\ust(t+1))_{\cF_t}.
\end{align*}

Note that under UCB-NCS, we have that $u(t)= K_{\theta(t)}(\ell(t))$. Let 
\begin{align}\label{def:k_max}
K_{\max}: = \sup_{\theta\in\Theta,\ell\in \{0,1\}} K_{\theta}(\ell), P_{\max}:= \sup_{\theta\in\Theta,\ell\in \{0,1\}} P_{\theta}(\ell).
\end{align}
After performing simple algebraic manipulations, we can show that
\begin{align}\label{ineq:r_3_b}
\cT_5  &\le P_{\max} \sum_{t=1}^{T-1} \left|\left( A\tht x(t)  +B\tht u(t) \right)^2\right.\notag\\
& \left.\qquad \qquad \qquad \qquad -\left( A\ust x(t) + B\ust u(t)\right)^{2}\right|\notag\\
& \le P_{\max}~ \cT_7^{1\slash 2}\times \cT_8^{1\slash 2}
\end{align}
where 
\begin{align*}
&\cT_7 := \sum_{t=1}^{T-1} \left| A\tht x(t)  -  A\ust x(t) + B\tht u(t) - B\ust u(t)  \right|^{2},\\
&\cT_8 := \sum_{t=1}^{T} \left|  A\tht x(t) +B\tht u(t) +  A\ust x(t) +B\ust u(t)  \right|^{2},
\end{align*}
and the last inequality in~\eqref{ineq:r_3_b} follows from Cauchy-Schwartz inequality. The terms $\cT_7,\cT_8$ are bounded in Lemma~\ref{lemma:bound_T4} and Lemma~\ref{lemma:bound_T5} in Appendix. We substitute these bounds in~\eqref{ineq:r_3_b} and obtain the following result.
\begin{lemma}\label{lemma:r_3_b}
On $\cE\cap \left(  \cH \cap \cJ\right)$, we have 
\begin{align}
\cT_5 &\le C_1 \sqrt{T}\log\left(V_1(T)\slash \lambda\right) \left(\gamma_1(\delta,T) +  \gamma_2(\delta,T)+2\lambda^{1\slash 2} \right)\notag\\
 &\times \sqrt{h(\delta,T)}~g^{3\slash 2}(\delta,T),
\mbox{ where, }\label{def:U_2}\\
C_1 :&= 2\sqrt{2}P_{\max}\left(1+K_{\max}\right) G_{cl,\max}\slash \lambda.\label{def:C_1}
\end{align}
\end{lemma} 
\vspace{.25cm}
It remains to bound $\cT_6$ in order to bound $R_2$. This is done in Lemma~\ref{lemma:t_6} of Appendix. 
\begin{lemma}\label{lemma:r_2}
Let 
\begin{align*}
&\cU_2:=  C_1 \sqrt{T}\log\left(V_1(T)\slash \lambda\right) \left(\gamma_1(\delta,T) +  \gamma_2(\delta,T)+2\lambda^{1\slash 2} \right)\\
& \times \sqrt{h(\delta,T)}~g^{3\slash 2}(\delta,T) \\
&+ P_{\max}\left(  G^{2}_{cl,\max}~g(\delta,T) + \sigma^2 \right)\sqrt{\alpha T \log T}.
\end{align*}
On $\cE\cap  \cH \cap \cJ$, we have $R_2\le \cU_2$.
\end{lemma}
\begin{proof}
Follows by substituting bounds on $\cT_5$ and $\cT_6$ from Lemma~\ref{lemma:r_3_b} and Lemma~\ref{lemma:t_6} into~\eqref{def:R_2}. 
\end{proof}
\section{Main Result}
\begin{theorem}[Bound on Regret]
Consider the NCS operating under UCB-NCS described in Algorithm~\ref{algo:ucb}.
Under Assumption~\ref{assum:clg}, $R(T)\le \cU_1 +\cU_2$ with a probability at least $7\delta +T^{2}\slash L^{\alpha}$. The terms $\cU_1,\cU_2$ are defined in~\eqref{def:U_1} and~\eqref{def:U_2} respectively. Upon ignoring terms and factors that are $O(\log T)$, this bound simplifies to 
\begin{align*}
  \sqrt{T}\left(\log^{1\slash 4}(1\slash \delta)+\sqrt{\alpha} P_{\max}G^{2}_{cl,\max} \log^{1\slash 2}(1\slash \delta)+ C_1\right).
\end{align*}
\end{theorem}
\begin{proof}
It follows from Lemma~\ref{lemma:bellman_regret} that $R(T)\le R_1+R_2$ on $\cE$.
Proof then follows by substituting upper-bounds from Lemma~\ref{lemma:r_1_b}, Lemma~\ref{lemma:r_2}, and using union bound to lower-bound the probability of $\cG \cap \cE \cap \cH \cap \cJ$.
\end{proof}
\section{Conclusion and Future Work}
We propose UCB-NCS, an adaptive control law, or learning rule for NCS, and provide its finite-time performance guarantees. We show that with a high probability, its regret scales as $\tilde{O}(\sqrt{T})$ upto constant factors. We identify a certain quantity which we call margin of stability of NCS. Regret increases with a smaller margin, which indicates a low network quality. \par
Results in this work can be extended in various directions. So far we considered only scalar systems. A natural extension is to the case of vector systems. Another direction is to derive lower-bounds on expected value of regret that can be achieved under any admissible control policy. 
\appendix\label{appendix}
\begin{lemma}[Bounding $\cT_7$]\label{lemma:bound_T4}
On $\cE \cap \left( \cH \cap \cJ \right)$, we have
\begin{align*}
\cT_7 &\le  \left(1+K_{\max}\right)^{2} \left(\gamma_1(\delta,T)+\gamma_2(T) + 2\lambda^{1\slash 2} \right)^{2} \\
&\times 2\left\{2\sqrt{h(\delta,T)}\right\}^{2}  \cdot g^{2}(\delta,T)\cdot \frac{1}{\lambda} \log\left(\frac{V_1(T)}{\lambda}\right).
\end{align*}
\end{lemma}
\begin{proof}
Let $\tau$ be the time step at which the latest episode begins. Since under UCB-NCS we have $u(t)=  K_{\theta(t)}(\ell(t)) x(t)$, it can be shown that
\begin{align}\label{ineq:adhoc}
&\left| \left( A\tht x(t)  -  A\ust x(t) \right) + \left( B\tht u(t) - B\ust u(t)  \right) \right| \notag\\
&\le \left| \left( A\tht  -  A\ust \right) \right| \left| x(t) \right|+ \left| \left( B\tht  - B\ust\right)\right| K_{\max}\left| x(t) \right|.
\end{align}
Now consider the following inequality,
\begin{align}\label{ineq:triangle}
\left| \left( A\tht  -  A\ust \right) \right| \le \left| A\tht  - \hat{A}(t)  \right| + \left|\hat{A}(t) - A\ust  \right|.
\end{align}
For $\theta = \left(A,B,p\right)\in \cC(\tau)$, we have,  
\begin{align}\label{ineq:3}
\left| A - \hat{A}(\tau) \right| |x(t)| &\le \sqrt{V_1(\tau)} \left| A - \hat{A}(\tau) \right| \frac{|x(t)| }{\sqrt{V_1(t)}} \sqrt{h(\delta,T)}\notag\\
&\le \left(\gamma_1(\tau)+\lambda^{1\slash 2} \right)\frac{|x(t)| }{\sqrt{V_1(t)}} \sqrt{h(\delta,T)},
\end{align}
where the first inequality follows from Lemma~\ref{lemma:t_4}, and second inequality follows from the size of the confidence intervals~\eqref{def:ucb_ci}. On $\cE$, we have $\theta\ust\in \cC(\tau)$, and also $\theta(t)\in \cC(\tau)$; so we use inequality~\eqref{ineq:3} with $\theta$ set equal to $\theta\ust,\theta(t)$, and combine the resulting inequalities with~\eqref{ineq:triangle} in order to obtain the following,
\begin{align}\label{ineq:1}
&\left| \left( A\tht  -  A\ust \right) \right| \left| x(t) \right| \notag\\
&\le  2\sqrt{h(\delta,T)} \left(\gamma_1(\delta,t)+\lambda^{1\slash 2} \right)\frac{|x(t)| }{\sqrt{V_1(t)}}.
\end{align}
A similar bound can be obtained for $\left| \left( B\tht  -  B\ust \right) \right|\left| x(t)\right|$ also. 
Remaining proof comprises of substituting these bounds in~\eqref{ineq:adhoc} and performing algebraic manipulations. We also utilize Lemma~10 of~\cite{abbasi2011regret} in order to bound  $\sum_{t=1}^{T}\left[|x(t)|^2\slash V_1(t) \wedge 1 \right], \sum_{t=1}^{T}\left[|x(t)|^2\slash V_2(t) \wedge 1 \right]$. 
\end{proof}

\begin{lemma}[Bounding $\cT_8$]\label{lemma:bound_T5}
On $ \cH \cap \cJ$, we have
\begin{align*}
\cT_8 \le G^{2}_{cl,\max}~T g(\delta,T),
\end{align*}
where
\begin{align*}
G_{cl,\max} :=  \sup_{\theta\in\Theta,\ell\in \{0,1\} }\left\{ |  A_{\theta} +B_{\theta} K_{\ell}(\theta)|, | A\ust  +B\ust K_{\ell}(\theta) |  \right\}
\end{align*}
\end{lemma}
\begin{proof}
Follows from Lemma~\ref{lemma:bounded_xt}.
\end{proof}
\begin{lemma}\label{lemma:t_6}
On $\cE\cap \left(\cH \cap \cJ \right)$ we have 
\begin{align*}
\cT_6 \le P_{\max}\left(  G^{2}_{cl,\max} g(\delta,T)  + \sigma^2 \right)\sqrt{\alpha T \log T}.
\end{align*}
\end{lemma}
\begin{proof}
We have
\begin{align*}
\cT_6 &\le \sum_{t=1}^{T-1} |p\tht - p\ust | P_{\max}\left(  G^{2}_{cl,\max} x^{2}(t)  + \sigma^2 \right) \\
&\le P_{\max}\left(  G^{2}_{cl,\max} \max_{t\in [1,T]} x^{2}(t)  + \sigma^2 \right)\left( \sum_{t=1}^{T-1} |p\tht - p\ust | \right).
\end{align*}
The proof is completed by noting that on $\cE$, we have $|p\tht - p\ust | \le \beta_1(t)$, while on $\cH \cap \cJ $ we have $\max_{t\in [1,T] } |x(t)| \le g(\delta,T)$.
\end{proof}
\begin{lemma}[Bounding $N(T)$]\label{lemma:bount_nt}
Define
\begin{align}\label{def:f}
&f(\delta,T) := \log\left( 1 + Tg^{2}(\delta,T)\slash \lambda\right)\notag\\
& +  \log\left( 1 + TK^{2}_{\max}g^{2}(\delta,T)\slash \lambda \right) + \log\left(T\right) .
\end{align}
We have that 
\begin{align*}
N(T) \le f(\delta,T)  \mbox{ on } \cH \cap \cJ.
\end{align*}
\end{lemma}
\begin{proof}
Recall that a new episode starts only when either a) $V_1(t)$ or $V_2(t)$ doubles, or b) samples used for estimating channel reliability double. Let $N_1(T),N_2(T)$ denote the number of episodes that began due to doubling of $V_1(t),V_2(t)$ respectively. Let $N_3(T)$ be number of episodes that began due to b). 
Clearly, $V_1(T) \ge 2^{N_1(T)} \lambda$, while on $ \cH \cap \cJ$ we have $|x(t)|\le g(\delta,T)$ (Lemma~\ref{lemma:bounded_xt}) so that $V_1(T) \le \lambda + Tg^{2}(\delta,T)$. Combining these, we obtain $N_1(T) \le \log\left( 1 + Tg^{2}(\delta,T)\slash \lambda \right)$. Similarly, $N_2(T) \le \log\left( 1 + TK^{2}_{\max}g^{2}(\delta,T)\slash \lambda \right)$. Also, $N_3(T) \le \log\left(T\right)$. The proof then follows by noting that $N(T) = N_1(T)+N_2(T)+N_3(T)$.
\end{proof}

\begin{lemma}[Bounding fluctuations within an episode]\label{lemma:t_4}
We have 
\begin{align*}
\cT_4 &\le P_{\max}  f(\delta,T) \cdot g(\delta,T),~ \forall \omega \in \cH \cap \cJ, \mbox{ and},\\
\cT_2 &\le 2P_{\max} g^{2}(\delta,T),~ \forall \omega \in \cH \cap \cJ,
\end{align*}
where $g(\delta,T),f(\delta,T)$ are as in~\eqref{def:upper_bound_expr},~\eqref{def:f}.
\end{lemma}
\begin{proof}
Recall $\cT_4 = \sum\limits_{t=1}^{T-1} V_{\theta(t)}(x(t),\ell(t)) - V_{\theta(t-1)}(x(t),\ell(t))$. Hence, the term in summation corresponding to time $t$ is non-zero only if the UCB estimate $\theta(t)$ changes, i.e., a new episode begins at time $t$. Thus,
\begin{align*}
\cT_4 \le N(T) P_{\max} \max_{t\in [1,T] } |x(t)|^2.
\end{align*}
The claim then follows by substituting bounds on $N(T)$ and $\max_{t\in [1,T]}|x(t)|$ from Lemma~\ref{lemma:bounded_xt} and Lemma~\ref{lemma:bount_nt}.
\end{proof}

\begin{lemma}\label{lemma:bound_fluctuations_xt}
Define,
\begin{align}\label{def:h}
h(\delta,T):= \max\left\{ 1 + 2L\left( 1+ \frac{1}{\lambda}\left[\frac{\log^{1\slash 2}\left(T\slash \delta\right)}{1-\exp(-\eta)}\right]^{2} \right),2\right\}.
\end{align}
We then have that
\begin{align*}
\frac{V_1(t)}{V_1(\tau_k)} \le h(\delta,T),~\forall t\in \left[\tau_k,\tau_{k+1}-1\right].
\end{align*}
Note that we have suppressed its dependence upon $\eta,L$ in order to simplify the notation.
\end{lemma}
\begin{proof}
Consider the following cases.\par
Case a): $t\in [\tau_k,\tau_k +L]$. We have
\begin{align*}
&\frac{V_1(t)}{V_1(\tau_k)}  -1=  \sum\limits_{s=\tau_k +1}^{t} x^{2}(s) \slash V_1(\tau_k)\\
& \le  L \left(\max_{s\in [\tau_k,\tau_k +L]}x(s) \right)^{2}\slash V_1(\tau_k)\\
& \le  L \left(|x(\tau_k)| +\frac{\log^{1\slash 2}\left(T\slash \delta\right)}{1-\exp(-\eta)} \right)^{2}\slash V_1(\tau_k)\\
& \le 2L\left( \frac{|x(\tau_k)|^{2}}{V_1(\tau_k)} + \frac{1}{V_1(\tau_k)}\left[\frac{\log^{1\slash 2}\left(T\slash \delta\right)}{1-\exp(-\eta)}\right]^{2} \right)\\
& \le 2L\left( \frac{|x(\tau_k)|^{2}}{|x(\tau_k)|^{2}} + \frac{1}{\lambda}\left[\frac{\log^{1\slash 2}\left(T\slash \delta\right)}{1-\exp(-\eta)}\right]^{2} \right). 
\end{align*}
Case b): $t\in t\in [\tau_k +L+1,\tau_{k+1}-1]$. In this case we have $\frac{V_1(t)}{V_1(\tau_k)}<2$, since a new episode begins
once the ratio becomes greater than or equal to $2$. 
\end{proof}

\bibliographystyle{ieeetran}
\bibliography{lcss}

\begin{thebibliography}{10}
\providecommand{\url}[1]{#1}
\csname url@samestyle\endcsname
\providecommand{\newblock}{\relax}
\providecommand{\bibinfo}[2]{#2}
\providecommand{\BIBentrySTDinterwordspacing}{\spaceskip=0pt\relax}
\providecommand{\BIBentryALTinterwordstretchfactor}{4}
\providecommand{\BIBentryALTinterwordspacing}{\spaceskip=\fontdimen2\font plus
\BIBentryALTinterwordstretchfactor\fontdimen3\font minus
  \fontdimen4\font\relax}
\providecommand{\BIBforeignlanguage}[2]{{%
\expandafter\ifx\csname l@#1\endcsname\relax
\typeout{** WARNING: IEEEtran.bst: No hyphenation pattern has been}%
\typeout{** loaded for the language `#1'. Using the pattern for}%
\typeout{** the default language instead.}%
\else
\language=\csname l@#1\endcsname
\fi
#2}}
\providecommand{\BIBdecl}{\relax}
\BIBdecl

\bibitem{bellman2015adaptive}
R.~E. Bellman, \emph{Adaptive control processes: a guided tour}.\hskip 1em plus
  0.5em minus 0.4em\relax Princeton university press, 2015.

\bibitem{kumar}
{P. R. Kumar and P. Varaiya}, \emph{{Stochastic systems: Estimation,
  identification and adaptive control}}.\hskip 1em plus 0.5em minus 0.4em\relax
  Prentice Hall Inc., Englewood Cliffs, 1986.

\bibitem{becker1985adaptive}
A.~Becker, P.~R. Kumar, and C.-Z. Wei, ``Adaptive control with the stochastic
  approximation algorithm: Geometry and convergence,'' \emph{IEEE Transactions
  on Automatic Control}, vol.~30, no.~4, pp. 330--338, 1985.

\bibitem{chen1987optimal}
H.-F. Chen and L.~Guo, ``Optimal adaptive control and consistent parameter
  estimates for armax model with quadratic cost,'' \emph{SIAM Journal on
  Control and Optimization}, vol.~25, no.~4, pp. 845--867, 1987.

\bibitem{bittanti2006adaptive}
S.~Bittanti, M.~C. Campi \emph{et~al.}, ``Adaptive control of linear time
  invariant systems: the ?bet on the best? principle,'' \emph{Communications in
  Information \& Systems}, vol.~6, no.~4, pp. 299--320, 2006.

\bibitem{abbasi2011regret}
Y.~Abbasi-Yadkori and C.~Szepesv{\'a}ri, ``Regret bounds for the adaptive
  control of linear quadratic systems,'' in \emph{Proceedings of the 24th
  Annual Conference on Learning Theory}, 2011, pp. 1--26.

\bibitem{lai1985asymptotically}
T.~L. Lai and H.~Robbins, ``Asymptotically efficient adaptive allocation
  rules,'' \emph{Advances in applied mathematics}, vol.~6, no.~1, pp. 4--22,
  1985.

\bibitem{auer2002finite}
P.~Auer, N.~Cesa-Bianchi, and P.~Fischer, ``Finite-time analysis of the
  multiarmed bandit problem,'' \emph{Machine learning}, vol.~47, no. 2-3, pp.
  235--256, 2002.

\bibitem{costa2006discrete}
O.~L.~V. Costa, M.~D. Fragoso, and R.~P. Marques, \emph{Discrete-time Markov
  jump linear systems}.\hskip 1em plus 0.5em minus 0.4em\relax Springer Science
  \& Business Media, 2006.

\bibitem{abbasi2011online}
Y.~Abbasi-Yadkori, D.~P{\'a}l, and C.~Szepesv{\'a}ri, ``Online least squares
  estimation with self-normalized processes: An application to bandit
  problems,'' \emph{arXiv preprint arXiv:1102.2670}, 2011.

\bibitem{tao2011random}
T.~Tao, V.~Vu \emph{et~al.}, ``Random matrices: universality of local
  eigenvalue statistics,'' \emph{Acta mathematica}, vol. 206, no.~1, pp.
  127--204, 2011.

\end{thebibliography}

\end{document}